\documentclass[conference]{IEEEtran}
\usepackage{times}

\usepackage[numbers]{natbib}
\usepackage{multicol}
\usepackage[bookmarks=true]{hyperref}

\usepackage{times}
\usepackage{epsfig}
\usepackage{graphicx}
\usepackage{subcaption}
\usepackage{amsmath}
\usepackage{amssymb}
\usepackage{amsfonts}
\usepackage{color}
\usepackage{algorithm}
\usepackage{algorithmic}

\newtheorem{theorem}{Theorem}[section]

\newtheorem{lemma}[theorem]{Lemma}
\newtheorem{problem}{Problem}

\begin{document}


\title{View Selection with Geometric Uncertainty Modelling}


\author{\authorblockN{Cheng Peng}
\authorblockA{College of Science and Engineering\\
University of Minnesota,\\
Minneapolis, MN, 55414\\
Email: peng0175@umn.edu}
\and
\authorblockN{Volkan Isler}
\authorblockA{College of Science and Engineering\\
University of Minnesota,\\
Minneapolis, MN, 55414\\
Email: isler@umn.edu}}


%

\maketitle

\begin{abstract}
  Estimating positions of world points from features observed in
  images is a key problem  in 3D reconstruction, image mosaicking,
  simultaneous localization and mapping and structure from motion. We
  consider a special instance in which there is a dominant ground plane $\mathcal{G}$ viewed from a parallel viewing plane $\mathcal{S}$ above it.
  Such instances commonly arise, for example, in aerial
  photography.
  
  Consider a world point $g \in \mathcal{G}$ and its worst case
  reconstruction uncertainty $\varepsilon(g,\mathcal{S})$ obtained by
  merging \emph{all} possible views of $g$ chosen from $\mathcal{S}$.  We
  first show that one can pick two views $s_p$ and $s_q$ such
  that the uncertainty $\varepsilon(g,\{s_p,s_q\})$ obtained using
  only these two views is almost as good as (i.e. within a small constant factor of)
  $\varepsilon(g,\mathcal{S})$. Next, we extend the result to the
  entire ground plane $\mathcal{G}$ and show that one can pick a small
  subset of $\mathcal{S'} \subseteq \mathcal{S}$ (which grows only
  linearly with the area of $\mathcal{G}$) and still obtain a constant
  factor approximation, for every point $g \in \mathcal{G}$, to the
  minimum worst case estimate obtained by merging all views in
  $\mathcal{S}$. 
  Finally, we present a multi-resolution view selection method which extends our techniques to non-planar scenes.
 We show that the method can
    produce rich and accurate dense reconstructions with a small number of views.

  Our results provide a view selection mechanism with provable
  performance guarantees which can drastically increase the speed of
  scene reconstruction algorithms. In addition to theoretical results,
  we demonstrate their effectiveness in an application where aerial
  imagery is used for monitoring farms and orchards.
\end{abstract}

\IEEEpeerreviewmaketitle

\section{Introduction}
Consider a scenario
where a plane flying at a fixed altitude is capturing images of a
ground plane below so as to reconstruct the scene (Figure~\ref{fig:denseComp10}).  Over the
course of its flight, the plane may capture thousands of
images which can easily overwhelm image reconstruction algorithms.
Our goal in this paper is to answer the question of whether we can select a
small number of images and focus only on them without reducing the
reconstruction quality. 


\begin{figure}
\centering
	\includegraphics[width=1\columnwidth]{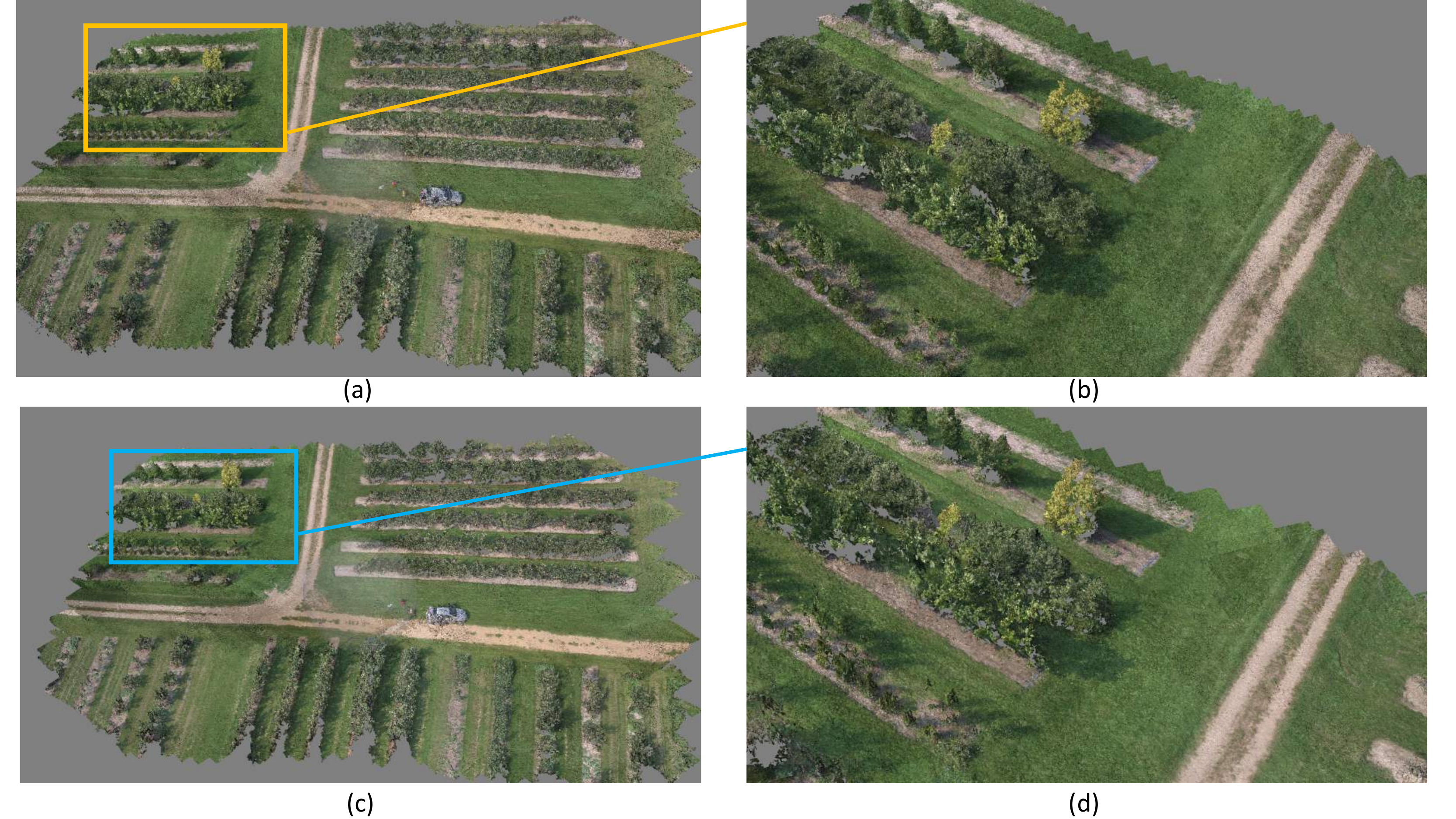}
	\caption{Comparison of dense reconstruction of the orchard from images taken at 10 meters altitude. (a) Dense Reconstruction using 893 images (b) Closeup view of the detailed reconstruction of the tree rows (c) Dense Reconstruction using 266 images extracted using our multi-resolution view selection method (d) Closeup view of the same tree row. }
\label{fig:denseComp10}
\end{figure} 

We first study a basic version where we focus on a single world point.
The goal is to select a small number of images from which the 3D
position of the world point can be accurately estimated
(Problem~\ref{prob:singlep}).  We then present a general version where
the goal is to minimize the error for the entire scene
(Problem~\ref{prob:multip}) from a small set of images. Note that in
the latter case, the same set of images must be used for every scene
point. We also extended our approach to a multi-resolution view selection 
scheme to accommodate non-planar scenes. 

In order to formalize these two problems, we first need to formalize
the error model and the uncertainty objective.  Let $g$ be a world point and $I$ be an
image taken from a camera at position $s$ and orientation $\theta$.
Let $p$ be the observed projection of $g$ onto $I$ and $p^*$ be the
unobserved true projection represented as vectors originating from the
camera center $s$. We will employ a {\em bounded uncertainty model}
where we will assume that the angle between $p$ and $p^*$ is bounded
by a known (or desired) quantity $\alpha$. 
Therefore, the 3D location
of the world point $g$ is contained inside a cone $C$ apexed at $s$
and with symmetry axis along $p$ and cone angle $2 \alpha$.
See Figure~\ref{fig:cone}.

\textit{Merging measurements:} In order to estimate the true location
of a world point from multiple measurements, we simply intersect the
corresponding cones. The diameter of the intersection is used as an
uncertainty measure. We chose diameter over the volume so as to avoid
degenerate cases where the intersection has almost zero volume but
large diameter which could still generate large triangulation error.

\textit{Uncertainty as worst-case reconstruction error:} 
Rather than associating a single cone for a specific measurement, our formulation considers a possibly infinite set of viable cones for a given true camera pose and world point pair. To do this, we consider all possible perturbations of relevant quantities (projection, location or pose).
When merging measurements, we consider the worst-case scenario  which maximizes the reconstruction uncertainty. 
 This formulation gives us a deterministic worst-case error model. It also allows us to factor out unknown or uncontrollable quantities such as camera orientation.
 
 \begin{figure}
\centering
	\includegraphics[width=0.5\textwidth]{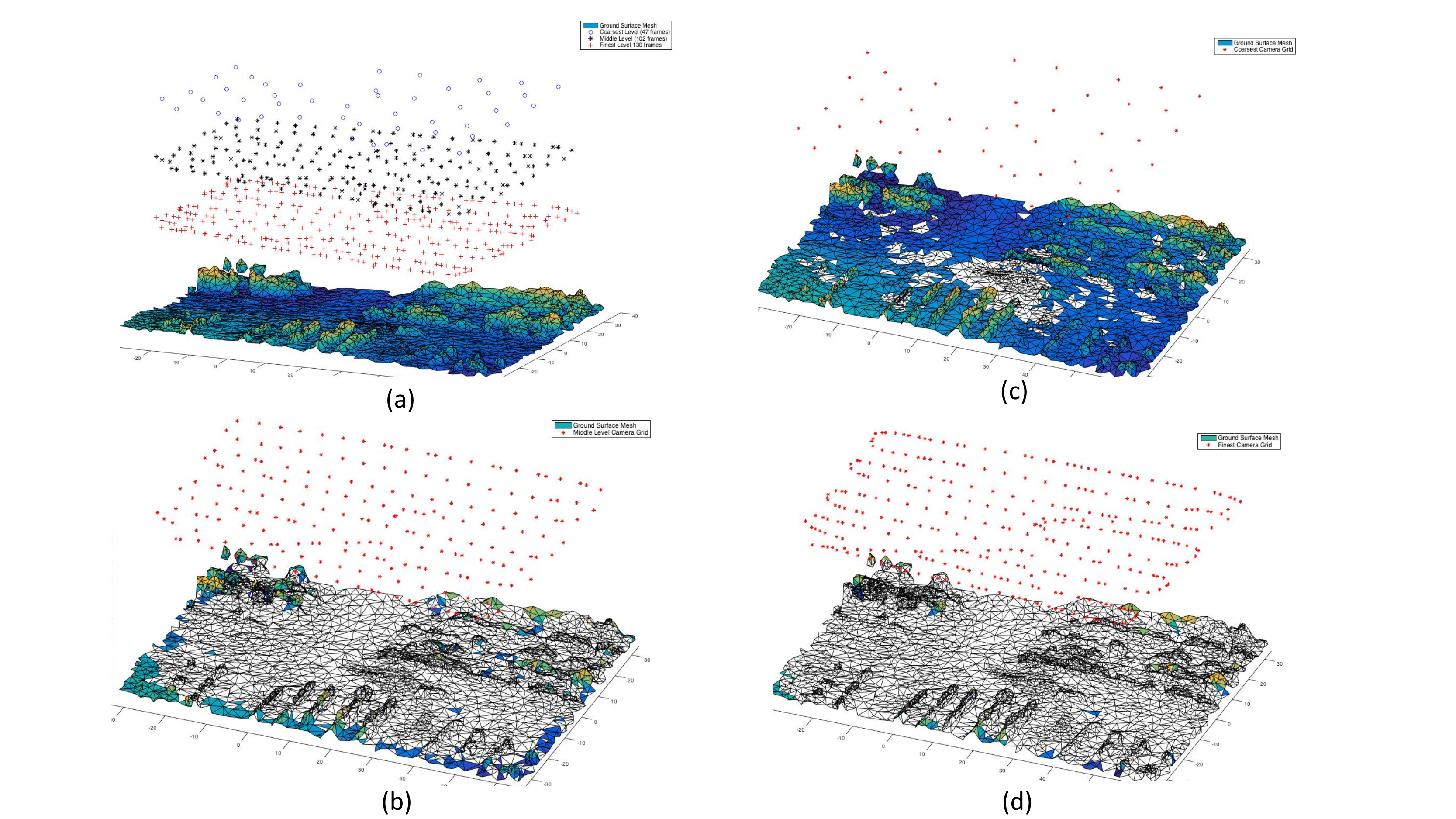}
	\caption{View Selection at Multiple resolution to cover the mesh region, where the color is the height and the white region is the covered region at each level: (a) View Selection at 3 resolution shown in blue, black and red. (b) View Selection at the Coarsest Level (c) View Selection at the Middle Level (d) View Selection at the Finest Level. Note that the coarser views cover partial planar region while the finer selection populates the more complex regions }
\label{fig:viewselection}
\end{figure} 

\section{Contributions and Related Work}
The importance of view selection for scene reconstruction is well
established.  One of the first view selection schemes for multi-view
stereo is presented in~\cite{farid1994view}.  The work of Maver and
Bajcsy~\cite{maver1993occlusions} and Kutulakos and Dyer~\cite{kutulakos1994recovering} use
contour information to choose viewing locations. 
A 2003 paper by Scott et al.~\cite{scott2003view} surveys
view selection methods.  Recently, Furukawa et
al.~\cite{furukawa2010towards} proposed a view selection scheme to
enable large scale 3D reconstruction.  Their method relies on
clustering images based on overlap. The resulting optimization problem
is solved iteratively.  The method of Hornung et
al.~\cite{hornung2008image} incrementally selects images and uses a
proxy to ensure coverage.  Mauro et al. resort to linear programming
to solve the view selection problem~\cite{mauro2014integer}. Sub-modular optimization~\cite{krause2014submodular}
has also been considered to jointly optimize the coverage and accuracy. 
However, it requires repeated visit of the same region. Both \cite{krause2014submodular} and \cite{hoppe2012photogrammetric}
uses surface meshes as geometrical reference to reason about optimal view selection. 
View selection has also been involved in image based modeling~\cite{vazquez2003automatic}, object
retrieval~\cite{gao2011less} and target
localization~\cite{isler2008sensor}. 

In the general reconstruction domain, key-frame methods \cite{klein2007parallel} \cite{murORB2} \cite{engel2014lsd} implement heuristics such as visible map features, distance between key-frames to decide if the current frame should be used for mapping. The main idea is to reduce the number of frames for bundle adjustment so as to make the system work in real-time. Mur-Artal et al.~\cite{murORB2} introduced the ``essential-graph" which builds a spanning tree from the image graph to achieve real-time performance. 
Snavely et al. \cite{snavely2008skeletal} proposed a method called ``skeleton set" that selects a subset of frames from the image graph to achieve similar reconstruction accuracy. However, they do not consider the geometry of the mapped environments. In Kaucic et al. \cite{planeBasedRecon}, the environment is assumed to be planar and the factorization method~\cite{sturm1996factorization} is used to speed up the bundle adjustment. 

In the present work, we consider an abstraction of the problem as: cameras on a viewing plane observing a planar world scene. We present a novel uncertainty model which allows us to characterize worst-case reconstruction error in a way that is independent of particular measurements. What differentiates our work from the previous body of work is that we present a  view selection mechanism with theoretical performance guarantees. Specifically, our \textbf{contributions} are the following. 
\begin{enumerate}
    \item We show that one can select two good views and obtain a reconstruction which is almost as good as merging all possible views from the entire viewing plane.
    \item We also show that a coarse camera grid (of resolution proportional to the scene depth) can provide a good reconstruction of the entire world plane.
    \item We present a multi-resolution view selection method which can be used for more general environments that are not strictly planar. 
\end{enumerate}

Our work is also related to error analysis in
stereo~\cite{sahabi1996analysis,cheong1998effects}. 
There are also many different uncertainty models. Bayram et al. \cite{Bayram2016sensor} models the bearing measurement's uncertainty as a function of linearized intersection area. Davison~\cite{davison2003real} approximates the uncertainty as a Gaussian distribution. 
We contribute to this line of work by analyzing the reconstruction error for two (best)
cameras with respect to the reconstruction error achievable by using
all possible cameras for the particular geometry we consider.

\section{Problem Definition}
 In this section, we introduce the general sensor selection problem.
Consider the world point $g \in \mathcal{G}$ and a camera $(s, \theta)$ where $s \in \mathbb{R}^3$ is the projection center and $\theta \in SO(3)$ is the  orientation.
Suppose we have a set of measurements $\{p_1, \ldots, p_k\}$ where each $p_i$ is expressed as a unit vector pointing towards the observed pixel and anchored at the corresponding camera center.
We need a function $f(p_1,p_2,...,p_k) = \hat{g}$ that maps measurements to $\hat{g}$, the estimate of $g$. This way, we can define the estimation error to be $||g - \hat{g}||$ by choosing an error measure $||\cdot||$.

 \begin{figure}[h]
 \centering
 	\includegraphics[width=0.4\columnwidth]{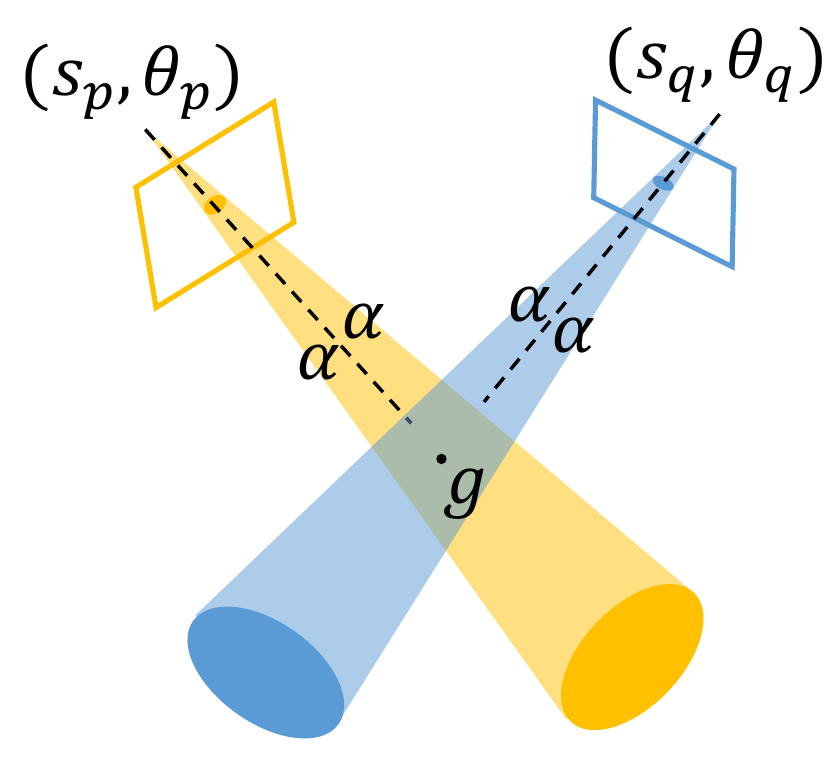}
 	\caption{Right circular cone for camera $(s,\theta)$ viewing target $g$}
 	\label{fig:cone}
 \end{figure}

In this paper, we will consider the following ``bounded uncertainty" characterization of the error:
 Consider the true measurement $p^* = Proj((s,\theta),g)$ given by the projection of $g$ onto camera $(s,\theta)$ which is also represented as a vector from $s$ pointing toward $g$. 
We make the assumption that the angle between the measurement $p$ and the true projection $p^*$ is bounded by a fixed threshold $\alpha$.
For a given measurement $p$, the rays corresponding to all  possible $p^*$ formulate to a cone denoted as $Cone_\alpha((s,\theta),p)$ as shown in Fig~\ref{fig:cone}, which is a function of  both the camera parameters $s$ and $g$ as well as  the measurement $p$.
For the rest of the paper, we will assume a fixed $\alpha$ and drop the subscript.
By intersecting the cones from multiple measurements $p_i$ from views $(s_i, \theta_i)$, we can get an estimate of the true target location.
The uncertainty is given by  the diameter of the intersection given by
$||\cap Cone((s_i,\theta_i),p_i)||$.

%
%
%
%


For sensor selection purposes, rather than  a single cone, it is beneficial to associate a set of cones for each measurement.
This will allow us to replace the randomness in the measurement process with a deterministic {\em worst-case analysis.}
To do this, for a given true  target location $g$ and a camera pose $(s, \theta)$, we generate $p^* = Proj((s,\theta),g)$. Then for every possible measurement $p$ within  angle $\alpha$ of $p^*$, we define $Cone((s,\theta),p)$ and include it with the set $S(g, s, \theta)$ associated with this world point/camera pair. Note that each cone in the set includes the true location $g$.
We can further eliminate the dependency on camera orientation by taking the union of these sets for each allowable orientation.
That is, we define $S(g, s) = \bigcup_\theta S(g, s, \theta)$  with the additional requirement that $g \in Cone((s,\theta),p)$ for each cone included in the union.


We can now define the worst case uncertainty for a given set $\mathcal{S} = \{s_1,s_2,...,s_k\}$  of camera centers and a ground point $g$ as:

$$\varepsilon(g, \mathcal{S}) = \max_{Cone_1 \in S(g,s_1), \ldots, Cone_k \in S(g, s_k)} ||\cap Cone_i ||$$

In other words, for each camera location $s_i$, a cone is chosen such that the chosen cones \emph{jointly} maximize the intersection diameter.
The advantage of this formulation is that since  the computation of $\varepsilon(g, \mathcal{S})$ implicitly generates all possible measurements for a given camera location and world point, it generates a worst case uncertainty independent of specific measurements and camera rotations.
We are now ready to define the first problem.

\begin{problem}
For a given world point $g$, the set of all possible viewpoints $\mathcal{S}$, a projection error bound $\alpha$, 
and an error tolerance parameter $\rho \in \mathbb{R}$,
choose a minimum cardinality subset $\mathcal{S'} \subseteq \mathcal{S}$, such that
$$\varepsilon(g,\mathcal{S'}) \leq \rho \varepsilon(g,\mathcal{S})$$
\label{prob:singlep}
\end{problem}

In Problem~\ref{prob:singlep} the goal is to choose a small subset of camera locations whose worst case uncertainty when reconstructing a given point $g$ is at most with a factor $\rho$ of the worst-case uncertainty of the entire viewing set. Problem~\ref{prob:multip} generalizes it to multiple points.

%

\begin{problem}
For a set of points $G \subseteq \mathcal{G}$, 
the set of all possible viewpoints $\mathcal{S}$, a projection error bound $\alpha$, 
and an error tolerance parameter $\phi \in \mathbb{R}$,
choose a minimum cardinality subset $\mathcal{S'} \subseteq \mathcal{S}$, such that
$$\max_{g \in {G}} \varepsilon(g,\mathcal{S'}) \leq \phi \max_{g \in \mathcal{G}} \varepsilon(g,\mathcal{S})$$
\label{prob:multip}
\end{problem}

In this paper, we study a specific geometric instance of these problems where  $\mathcal{G}$ and $\mathcal{S}$ are two parallel planes with distance $h$ apart. 
For a given $g \in \mathcal{G}$, we will define $\varepsilon_{\infty}(g) = \varepsilon(g,\mathcal{S})$.
%
%
\section{Sensor Selection for a Single Point}
In this section, we study Problem 1 where the goal is to choose cameras to reconstruct a single point.
We will start with the two dimensional (2D) case where the ground and viewing planes reduce to lines, and the uncertainty cones become wedges.

Our key result in this section is that for any point $g$, one can choose two cameras whose worst case uncertainty $\varepsilon_{2}(g)$ is almost as good as $\varepsilon_{\infty}(g)$ , which is the worst case uncertainty obtained by merging the views from {\em all} cameras. The key ideas in obtaining this result are: (1)~if we choose two cameras at locations $p$ and $q$ who view $g$ symmetrically at 90 degrees (i.e. $\angle p g q = \pi/2$), the diagonals of the worst-case uncertainty polygon (the intersection of the two wedges) are roughly of equal length. (2)~Any other camera added to the sensor set can be rotated to contain the horizontal diagonal. Therefore, it does not reduce the uncertainty drastically. 

%
%

\subsection{The Solution of Problem 1 in 2D}


Let $A = \arg \max(\varepsilon_\infty(g))$ be the set of wedges which yield the minimum worst case uncertainty. 
For every point $c$ on the viewing plane, there is a wedge in $A$ which (i)~ is apexed at $c$, (ii) has wedge angle $\alpha$ and (iii)~contains $g$.
By definition of $\varepsilon_\infty(g)$, the wedges are rotated so as to maximize the diagonal of the intersection.

%


\begin{theorem}\label{thrm:opt2cameras2dbound}
Consider a target $g$ on line $G$ and viewing set $S$ composed of all camera locations on $S$ parallel to $G$.
There exist two cameras $s_p$ and $s_q$ which guarantee that
\begin{equation}\label{eq:opt2cameras2dbound}
\varepsilon_2 \leq \sqrt{\frac{1+2\alpha}{1-4\alpha}}  \varepsilon_\infty
\end{equation}
where $\varepsilon_\infty = \varepsilon(g,S)$  is the minimum worst case uncertainty of the entire viewing set, and 
$\varepsilon_2$ is the worst case uncertainty of  $\{s_p, s_q \}$ and $0 \leq \alpha < 1/4$ is the error threshold measured in radians.
%
\end{theorem}



\begin{figure}[h]
\centering
	\includegraphics[width=0.5\textwidth]{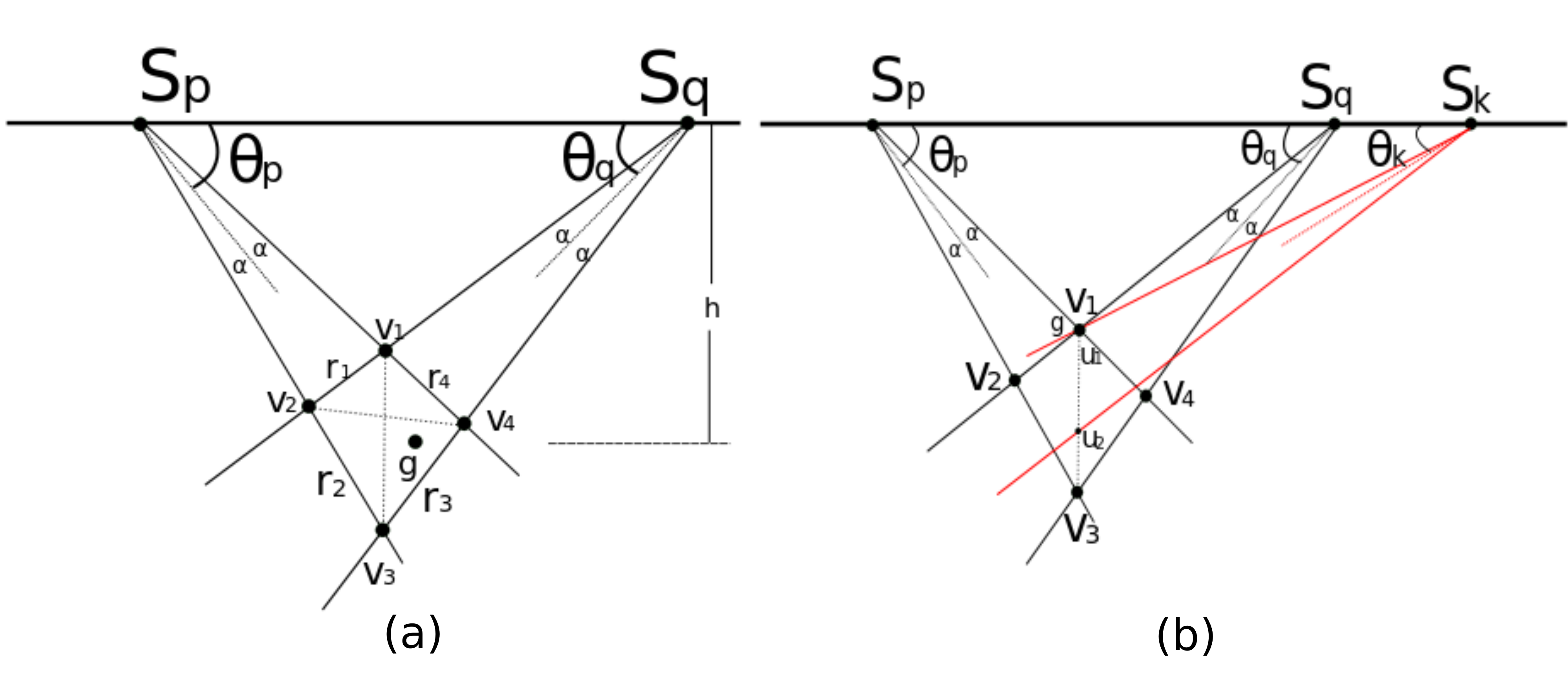}
	\caption{(a) Notation for the two camera selection $s_p$ and $s_q$ (b) If the cone created by $s_k$ that does not contain $diag_1$, we get a contradiction (proof of Lemma~\ref{lem:diag1max})}
	\label{fig:campq}
\end{figure} 

We will prove the theorem directly by providing the two cameras, computing their worst-case uncertainty $\varepsilon_2$ and comparing it with the minimum possible worst-case uncertainty. 
First, we present the notation and the setup used in the computations.
We set a coordinate system whose origin is at the target $g$. The $x$-axis is on $G$ and the $z$-axis points ``up" toward the viewing plane.
The locations of the two cameras are chosen as:
 $s_p = [-t/2,h]$ and $s_q = [t/2,h]$ where $t = \frac{2h}{\tan(\pi/4-\alpha)}$ and the cone orientations $\theta_p, \theta_q$ respectively (Fig~\ref{fig:campq} (a)). We use the angle $\theta$ between the bisector of a wedge with respect to $S$ for orientation. Of the two half-planes whose intersection yields the wedge, the inner half plane is the one that is closer to $S$ -- i.e. the angle measured is smaller while the other half-plane is the outer half-plane also shown in Fig~\ref{fig:campq} (a). Note that $\theta_p,\theta_q \in [\pi/4-2\alpha,\pi/4]$.
 
Their worst case uncertainty is given by
\begin{equation}\label{eq:uncertainty2cam2d}
\varepsilon_2 = \max_{\theta_p,\theta_q}||Cone((s_p,\theta_p),g)\cap Cone((s_q,\theta_q),g)||
\end{equation}
Consider the two wedges which give the worst case uncertainty (i.e. $\arg \max$ of $\varepsilon_2$). 
Let $Q_{pq}$ be their intersection with vertices $\{v_1,v_2,v_3,v_4\}$ and  edges  $\{e_1,e_2,e_3,e_4\}$ (Fig~\ref{fig:campq} (a)).
The lengths of the edges are denoted as $r_i = ||e_i||$ and the length of the diagonals are denoted by $diag_1 = ||\overline{v_1v_3}||,diag_2 = ||\overline{v_2v_4}||$.

We now compute these quantities.

\subsubsection{Computing $\varepsilon_2$}
In order to maximize over the orientation, we first establish the closed form solution for the edges and diagonals as functions of $h$,$t$,$\theta_{p,q}$,and $\alpha$.

Using the law of cosines, $diag_1$ can be calculated as 
\begin{equation}\label{eq:diag1}
diag_1^2 = r_1^2+r_2^2
	-2  r_1  r_2  \cos(\theta_p+\theta_q)
\end{equation}
Similarly, the $diag_2$ can be calculated as 
\begin{equation}\label{eq:diag2}
diag_2^2 = r_1^2+r_4^2
	-2 r_1  r_4  \cos(\pi - \theta_p-\theta_q + 2\alpha)
\end{equation}
The detailed derivation is shown in Appendix~\ref{sec:wedgeIntersection}. 

We now consider the vertical diagonal 
whose length $diag_1$ is given in Equation~\ref{eq:diag1}. 
It is maximized when $\theta_p = \theta_q = \pi/4$.
Fig~\ref{fig:diag1_maximum} shows $diag_1$ 
as a function of the two wedge angles $\theta_p$ and $\theta_q$ and for $\alpha \leq 0.1$ rad. 
When $\theta_p = \theta_q = \pi/4$, the vertex $v_1 = g$, which means that the inner half-planes of $Cone_p$ and $Cone_q$ intersect at $g$.

\begin{figure}[h]
\centering
	\includegraphics[width=0.3\textwidth]{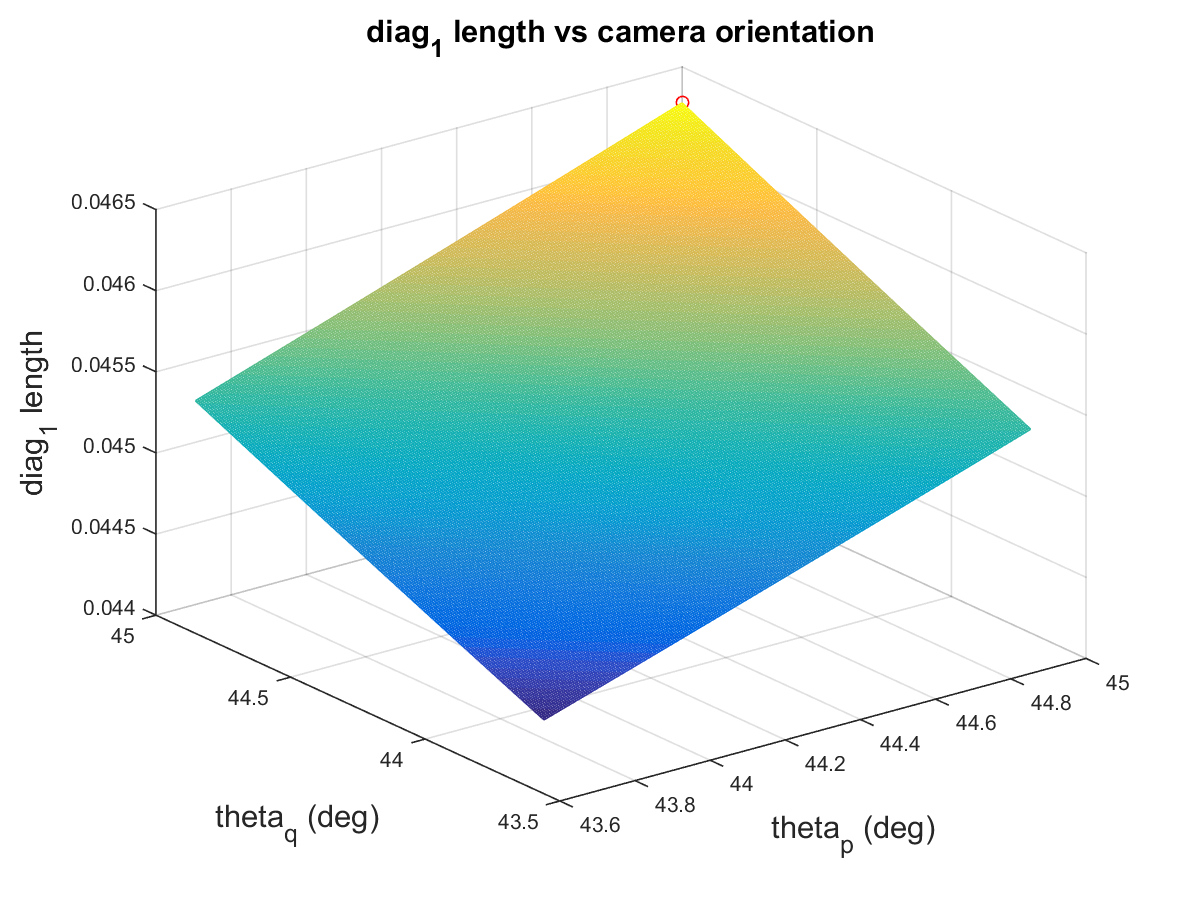}
	\caption{$diag_1$ length as a function of $\theta_p$ and $\theta_q$}
	\label{fig:diag1_maximum}
\end{figure} 

We can therefore set $\theta_p = \theta_q = \pi/4$ and write the equation of $diag_1$ as a function of 
$\alpha$ and $h$: Using the law of sines on the triangle $\triangle(s_qv_1v_3)$ and $\overline{v_1s_q} = h/\sin(\pi/4-\alpha)$, we obtain:
\begin{align*}
\frac{diag_1}{\sin(2\alpha)} &= \frac{\overline{v_1s_q}}{\sin(\frac{\pi}{2}-\theta-\alpha)} \\
diag_1 &= \frac{2h\sin(2\alpha)}{1-\sin(2\alpha)}
\end{align*}

This establishes the maximum length of the diagonal $diag_1 =\frac{2h\sin(2\alpha)}{1-\sin(2\alpha)}$ in the worst case configuration of $\theta_p=\theta_q = \pi/4$. 

We now compare $\varepsilon_2(s_p, s_q) = \max||Q_{pq}||$ with $\varepsilon_\infty$.

\begin{lemma}\label{lem:diag1max}
Consider the two cameras $s_p,s_q$ in the optimal configuration described above and let  $diag_1$
be the intersection of their worst-case uncertainty polygon $Q_{pq}$.
Any cone starting from location $s_k \in A-\{s_p,s_q\}$, can be rotated to an angle $\theta_k$ such that both $g$ and 
$diag_1$ are contained in its uncertainty wedge $Cone((s_k,\theta_k),g)$.
\end{lemma}


Now that we established that two cameras suffice, we compute the uncertainty value:
\begin{lemma}\label{lem:upperbound2camerasValue2d}
Given the two cameras $s_p,s_q$, the intersection polygon $Q_{pq}$, the maximum length of the diagonal $diag_1 =\frac{2h\sin(2\alpha)}{1-\sin(2\alpha)}$ when $\theta_p=\theta_q = \pi/4$, and the worst case uncertainty $\varepsilon_2 = \max||Q_{pq}||$.
\begin{equation}\label{eq:upperbound2camerasValue2d}
\varepsilon_2 \leq \sqrt{\frac{1+2\alpha}{1-4\alpha}} \cdot \frac{2h\sin(2\alpha)}{1-\sin(2\alpha)}
\end{equation}
\end{lemma}



Now we can conclude by presenting the proof of Theorem~\ref{thrm:opt2cameras2dbound}.

\begin{proof}
Combining Lemma~\ref{lem:diag1max} and Lemma~\ref{lem:upperbound2camerasValue2d}, we can conclude that $diag_1 \leq \varepsilon_\infty \leq diag_2$. Therefore, $\varepsilon_2 \leq \sqrt{\frac{1+2\alpha}{1-4\alpha}} \cdot \varepsilon_\infty$ 
\end{proof}

In this section, we showed that there exist two cameras $s_p$ and $s_q$ with orientation $\theta_p = \theta_q = \pi/4$ such that their worst case uncertainty $\varepsilon_2 \leq \sqrt{\frac{1+2\alpha}{1-4\alpha}} \cdot \varepsilon_\infty$.
We will call the pair of cameras $s_p,s_q$ as the \textbf{optimal pair} for the rest of the paper and this configuration as the \textbf{optimal configuration} of $\{s_p, s_q\}$.



\subsection{The Solution of Problem 1 in 3D}
The results of the previous section readily extend to  $\varepsilon_\infty$ in 3-D.  

\begin{theorem}\label{thrm:opt2cameras3dbound}
Given a target $g \in \mathcal{G}$ and a set of cameras $s \in \mathcal{S}$, where the distance between $\mathcal{G}$ and $\mathcal{S}$ is $h$ and the number of cameras in $\mathcal{S}$ is unbounded, we claim that the optimal pair $s_p$ and $s_q$ gives
\begin{equation}\label{eq:opt2cameras3dbound}
\varepsilon_2 \leq \sqrt{\frac{1+2\alpha}{1-4\alpha}} \cdot \varepsilon_\infty
\end{equation}
where the minimum worst case uncertainty in 3-D is $\varepsilon_\infty = \varepsilon(g,\mathcal{S})$ and worst case uncertainty from two cameras $s_p$ and $s_q$ is $\varepsilon_2$.
\end{theorem}

\begin{figure}[h]
\centering
	\includegraphics[width=0.3\textwidth]{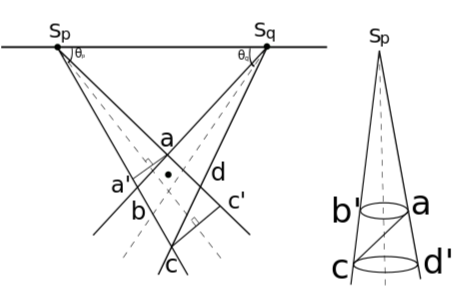}
	\caption{Uncertainty in 3D given by two intersecting cones }
	\label{fig:u3d2d}
\end{figure} 

To prove the theorem, all we have to do is to observe that the diagonal of a perpendicular cross section of the cone bounds the uncertainty in 3D as well. See Fig~\ref{fig:u3d2d}. Therefore, we can apply Theorem~\ref{thrm:opt2cameras2dbound}.


\section{Sensor Selection For the Entire Scene}
In the previous section, we established that for a world point $g$, the optimal pair of cameras can produce a reconstruction with approximation ratio less than $\sqrt{\frac{1+2\alpha}{1-4\alpha}}$ of the optimal reconstruction (Theorem~\ref{thrm:opt2cameras3dbound}).
However, if we use the dedicated pair directly for every scene point, we may end up choosing two cameras for each scene point, which in turn might result in a large number of cameras.
 \begin{figure}[h]
 \centering
 	\includegraphics[width=0.4\textwidth]{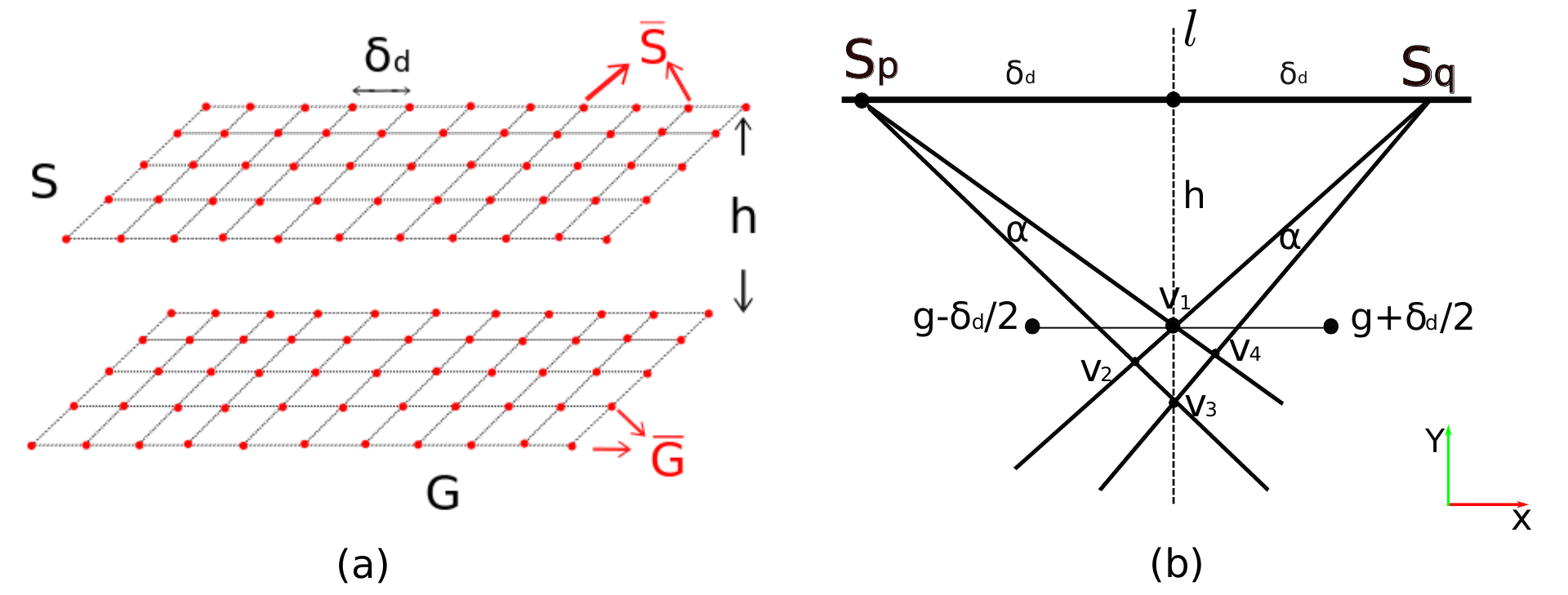}
 	\caption{(a) The square sensor grid in 3D (b) Square sensor grid in 2D with ground point variation.}
 	\label{fig:grid}
 \end{figure}
 
In this section, we show that a coarse grid of cameras provide a good reconstruction for every scene point.
Recall that $\mathcal{G}$ is the ground plane, $\mathcal{S}$ is the view plane,   $\mathcal{G}$ is parallel to  $\mathcal{S}$ and the distance between them is $h$. 
Let $\mathcal{\overline{S}}$ be a square grid imposed on $\mathcal{S}$ with resolution $\delta_d$ (Fig~\ref{fig:grid} (a)). The same grid $\mathcal{\overline{G}}$ is also imposed on the ground plane $\mathcal{G}$. 
To demonstrate the main strategy at a high-level, consider a ground point $g \in \mathcal{G}$, such that the optimal pair of cameras lies in camera grid $\mathcal{\overline{S}}$. We will show that the optimal pair of cameras can still provide ``good" reconstruction for all points in a region $R(g)$ around $g$. 

Using the result we will show in Theorem~\ref{thrm:grid3d2} that a constant number of cameras for a ground plane can be used to achieve a small approximation ratio.

\subsection{Problem 2 in 2D}
For cameras in the grid $s \in \overline{S}$ and target $g \in G$, we define the grid uncertainty $\overline{\varepsilon}(g)$ using only the best two cameras in grid $\overline{S}$ as the following
$$
\overline{\varepsilon}(g) = \min_{s_i, s_j \in \overline{S}}\varepsilon(g,\{s_i,s_j\}) 
$$
As mentioned earlier, we will choose the grid resolution to be $\delta_d = h$ for the following analysis.

Now, we define the geometry for Lemmas~\ref{lem:grid2d1}, \ref{lem:grid2d2}, ~\ref{lem:grid2dhori}, and \ref{lem:grid2dvert}. 
Let $g \in \overline{G}$ be a grid location with height $h$ to the viewing plane $S$. Now, we choose the optimal pair of cameras for the target $g$ as $\{s_p,s_q\} \in \overline{S}$ as shown in Fig~\ref{fig:grid} (b). Let $l$ be a line passing through $g$ with $l \perp G$ and $x = l \cap Cone(s)$, where $x$ is the intersection line segments between $l$ and the Cone generated by sensor $s$ and target $g$. 

In order to bound the uncertainty of any target $\forall g \in G$ using the camera grid $\overline{S}$, we need to explore the uncertainty of the targets in grid cells (Fig~\ref{fig:grid} (b)). Therefore, we fix a grid point and define a range of targets $R(g) = [g-\delta_d/2, g+\delta_d/2]$ such that $R(g)$ is generated by moving $g \in \overline{G}$ along the $x$-axis of the grid. We now show that the worst case uncertainty is achieved at the end points of this interval (i.e. the midpoint of two grid locations) bound by $max(||x_p||, ||x_q||)$, where $||x||$ represents the length of line segment of $x$. We define $diag_1 = \overline{ac}$ and $diag_2 = \overline{bd}$ in Fig~\ref{fig:u3d2d}.

\begin{lemma}\label{lem:grid2d1}
When $\theta_p + \theta_q \geq \frac{\pi}{2} + \alpha$, $diag_1 > diag_2$.  
\end{lemma}


\begin{lemma}\label{lem:grid2d2}
$\theta_p + \theta_q$ is maximized when the inner half-plane of both cones intersect $g^* = g \pm  \delta_d/2$. 
\end{lemma}


It is clear that either $||x_p||$ or $||x_q||$ is always larger or equal to $diag_1$, which can be used to generate the worst case bound.
\begin{theorem}\label{thrm:grid3d1}
For all targets $g \in G$ and sensor grid $\overline{S}$ with resolution $\delta_d = h$, the worst case grid uncertainty $\overline{\varepsilon}(g)$ using only two cameras from $\overline{S}$ is bounded as follows
$$
\overline{\varepsilon}(g) \leqslant 1.72 \varepsilon_\infty
$$
\end{theorem}

\subsection{Relaxing planar scene and viewing plane assumptions}
So far, our analyses of the uncertainty bound are based on the parallel plane assumptions. Such assumptions are reasonable for some applications such as high altitude aerial imagery. 

In this section, we relax these assumptions so that the theorem can be applied to more general environments. Define horizontal and vertical variation as $\lambda_vh, \lambda_hh$, where $0 < \lambda_v, \lambda_h < 1$.
We will analyze the change in $\overline{\varepsilon}(g)$ when adding variation in both horizontal and vertical directions.
The new camera location $\hat{s}$ is generated by perturbing $s$  by $\lambda_vh,\lambda_hh$ amount in vertical and horizontal directions.
We analyze both effects from vertical and horizontal variations in Appendix~\ref{sec:appendix} and get the following results. 

\begin{theorem}\label{thrm:gridU2d}
For all targets $g \in G$ and sensor grid $\overline{S}$ with resolution $\delta_d = h$ and variation $\lambda_v,\lambda_h$, the worst case grid uncertainty $\overline{\varepsilon}(g)$ using only two cameras from $\overline{S}$ is bounded as follows
$$
\overline{\varepsilon}(g) \leqslant 1.72\frac{1+\lambda_v}{1-\lambda_h} \varepsilon_\infty
$$
\end{theorem}
\begin{proof}
The result can be derived by combining Lemma~\ref{lem:grid2dhori}, Lemma~\ref{lem:grid2dvert} and Theorem~\ref{thrm:grid3d1}.
\end{proof}

We can see that small deviation from the camera position or the ground plane does not introduce significant uncertainty. 

\subsection{Problem 2 in 3D}

In 3D, we use the same grid resolution $\delta_d = h$ which is half of the distance between the optimal pair of cameras. The main result is
\begin{theorem}\label{thrm:grid3d2}
For all targets $g \in \mathcal{G}$ and sensor grid $\overline{S}$ with resolution $\delta_d = h$ and variation $\lambda_v,\lambda_h$, the worst case grid uncertainty $\overline{\varepsilon}(g)$ using only two cameras from $\overline{S}$ is bounded as follows
$$
\overline{\varepsilon}(g) \leqslant 2.47\frac{1+\lambda_v}{1-\lambda_h} \varepsilon_\infty
$$
\end{theorem}

The proof is similar to the 2D case. It is extended to include perturbations in both $x$ and $y$ directions which slightly increases the bounds, which is shown in Appendix~\ref{sec:appendix} Figure~\ref{fig:grid3dv2}.


\begin{figure}[h]
\centering
	\includegraphics[width=0.2\textwidth]{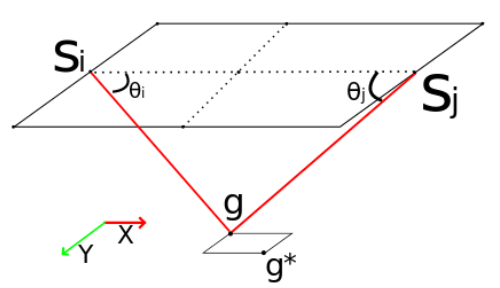}
	\caption{Camera grid in 3D: $g$ is perturbed to $g^*$ to achieve worst case uncertainty.}
	\label{fig:grid3dv2}
\end{figure} 

Theorem~\ref{thrm:grid3d2} allows us to bound the geometric error even in the presence of variations in both viewing and scene planes. However, it does not address visibility: variations in the scene can cause occlusions which can block camera views. In the next section, we address this issue. 

\section{Multi-Resolution View Selection}
In this section, we explore how to extend our previous camera view grid approach to non-planar regions such as orchards and forests. The parallel plane assumption can produce good results with high altitude, but will be insufficient to model non-planar regions. For this purpose, we propose a multi-resolution approach, which generates multiple camera view grids in a coarse to fine manner, to reconstruct more general regions. 

The input to our method is a surface mesh generated using sparse points clouds from a SLAM method such as  ORB-SLAM~\cite{murORB2}. It then outputs a subset of the views such that each face of the mesh can be \emph{well-covered}, that is, covered by at least 3 cameras separated by the current grid resolution. To ensure coverage quality, we double the grid resolution at each iteration so that the minimum distance between cameras is bounded. We present the details in Section~\ref{sec:viewselection}.

As the scene becomes more complex, the multi-resolution approach is able to adapt the terrain. For a given grid resolution, we iterate through all triangles and if they are well-covered by the current subset of views, those views will be added to the solution. However, the potential views that can see the triangle are limited due to occlusion and matching quality. Therefore, we introduce a visibility cone for each triangle in Section~\ref{sec:mesh}  to limit the search space.  

Similar to~\cite{krause2014submodular} and~\cite{hoppe2012photogrammetric}, we also generate scene meshes to reason about the geometry. The main difference of our work is that first, we do not require a secondary visit to the scene. The existing trajectory of views can be sufficient enough to cover the environment in most cases.  Second, we generalize the visibility for each triangle mesh such that well-covered views can be predicted instead of histogram method ~\cite{hoppe2012photogrammetric} that is strongly case sensitive. 

\subsection{Visibility Cone}
\label{sec:mesh}

\begin{figure}
    \centering
    \includegraphics[width=0.4\textwidth]{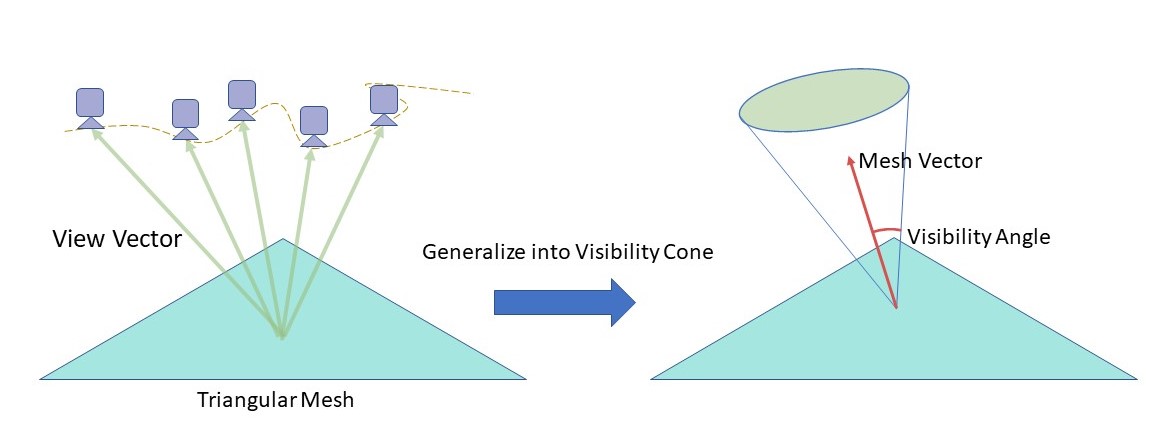}
    \caption{The visibility cone generated from visible cameras}
    \label{fig:visibilitycone}
\end{figure}

A camera is defined to be visible to a triangle mesh when it contains 2D feature of a point around the mesh. A viewing vector for a triangle is defined as the vector pointing from the center of the triangle to the corresponding camera as shown in Figure~\ref{fig:visibilitycone}. 
The mesh vector is then the average of all viewing vectors for that triangle mesh. We also define the visibility angle of each triangle as the average angle between all viewing vector. 
We can therefore predict the visibility of a triangle using both the visibility angle and the mesh vector. Essentially, we generate a visibility cone, where the direction of the cone is the mesh vector and the aperture is the visibility angle. We do not consider the effects of viewing angles since all the views are assumed to be facing downwards, which can be easily maintained with a gimbal stabilizer. 
Unlike the approaches from ~\cite{hoppe2012photogrammetric} that extract the histogram for each mesh triangle, we bound the region of possible visible camera views using the mesh visibility.

\subsection{Coarse to Fine View Selection}
\label{sec:viewselection}
After identifying the visibility cone for each triangle, we utilize our previous proposal of the camera grid in a coarse to fine manner.
\begin{figure}
    \centering
    \includegraphics[width=0.2\textwidth]{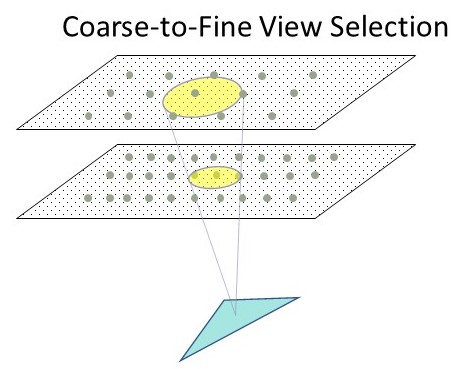}
    \caption{Multi-resolution view selection for each triangular mesh, where the camera views (only in one level) intersect with the visibility cone are added to the solution.}
    \label{fig:multireso}
\end{figure}

\begin{algorithm}
\caption{View Selection. Let $M = \{m_1,m_2,...\}$ be all triangle meshes and $J = \{s_1,s_2,...\}$ be all camera poses from the trajectory. Let $\pi(m_i,J)$ be the function that output all cameras in $J$ that are within the visibility cone of  $m_i$.}
\begin{algorithmic}
\REQUIRE set initial grid resolution $R$, set solution $sol=[]$
\WHILE{ when $M$ is not empty}
    \STATE {Pick camera grid $S_R \subseteq J$ with spacing $R$}
    \FORALL{$m_i \in M$}
                \STATE{$S = S_R \bigcap sol$}
        \IF{$|\pi(m_i,S)| \geq 3$}
            \STATE{$sol = \pi(m_i,S) \bigcap sol$}
            \STATE{remove $m_i$ from $M$}
        \ENDIF
    \ENDFOR
    \STATE{$R = R/2$};
\ENDWHILE
\STATE{Output final selected views $sol$}
\end{algorithmic}
\end{algorithm}

For a given grid resolution, we iterate through all faces of the mesh and check their visibility cones against current subset of views. For each face, if the visibility cone contains at least 3 camera views from the current subset of views, then those views will be added to the solution as shown in Figure~\ref{fig:multireso}. Those faces covered by 3 or more cameras will not be considered in the next iteration. To ensure the quality of the selected views, we impose that for each face, there are at least 3 views visible to the mesh so that feature matching error can be reduced. Since we also increase the grid resolution by two fold for each iteration, the chosen views for a specific mesh guarantee a minimum spacing. Giving the grid spacing $R$ at the first iteration, after $k$ iterations, the minimum spacing between all views will be $\frac{1}{2^k}R$ instead of arbitrarily small spacing that reduces reconstruction quality.

\section{Evaluation}

\begin{table*}[t]
\caption{The comparison of average reprojection error and reconstruction time for the two experiments}
\centering
\begin{tabular}{|c||c|c|c||c|c|c|}
\hline
 & Original Frames & Avg Reprj Err & SFM time (min) &  Camera Grid Frames & Avg Reprj Err & SFM time (min)\\
\hline
Orchard: 30 meters Flight & 416 & 0.842 & 313.6 & 76 & 0.934 & 4.1\\ 
\hline
Orchard: 10 meters Flight & 375 & 0.724 & 374.7 & 84 & 0.842 & 4.4\\ 
\hline
 & Original Frames & Avg Reprj Err & Dense Recon (min) &  Multi-Resol Method & Avg Reprj Err & Dense Recon (min)\\
\hline
Orchard: 30 meters Flight & 875 & 0.863 & 1463 & 209 & 0.931 & 115\\ 
\hline
Orchard: 10 meters Flight & 893 & 0.944 & 1522 & 266 & 1.243 & 167\\ 
\hline
\end{tabular}
\label{tab:reconT}
\end{table*}

In this section, we present simulation results used for validating the uncertainty model and results followed by a real-world reconstruction performance using the coarse to fine view selection method.

\subsection{Simulations}
We used the following parameters of a GOPRO HERO 3 for simulations.
Resolution: $1920 \times 1080$, Field of view: $120^\circ \times 70^\circ$.
The calibration error in pixels was  $[0.2061, 0.2183]$. 
For all simulations we used an iMac with $3.3$GHz quad-core Intel Core i5 and $16$GB of RAM. 

{\em Model justification:}
We consider the following sources of uncertainty: finite resolution, calibration errors, camera center location, camera orientation.

\begin{figure}
\centering
	\includegraphics[width=0.5\textwidth]{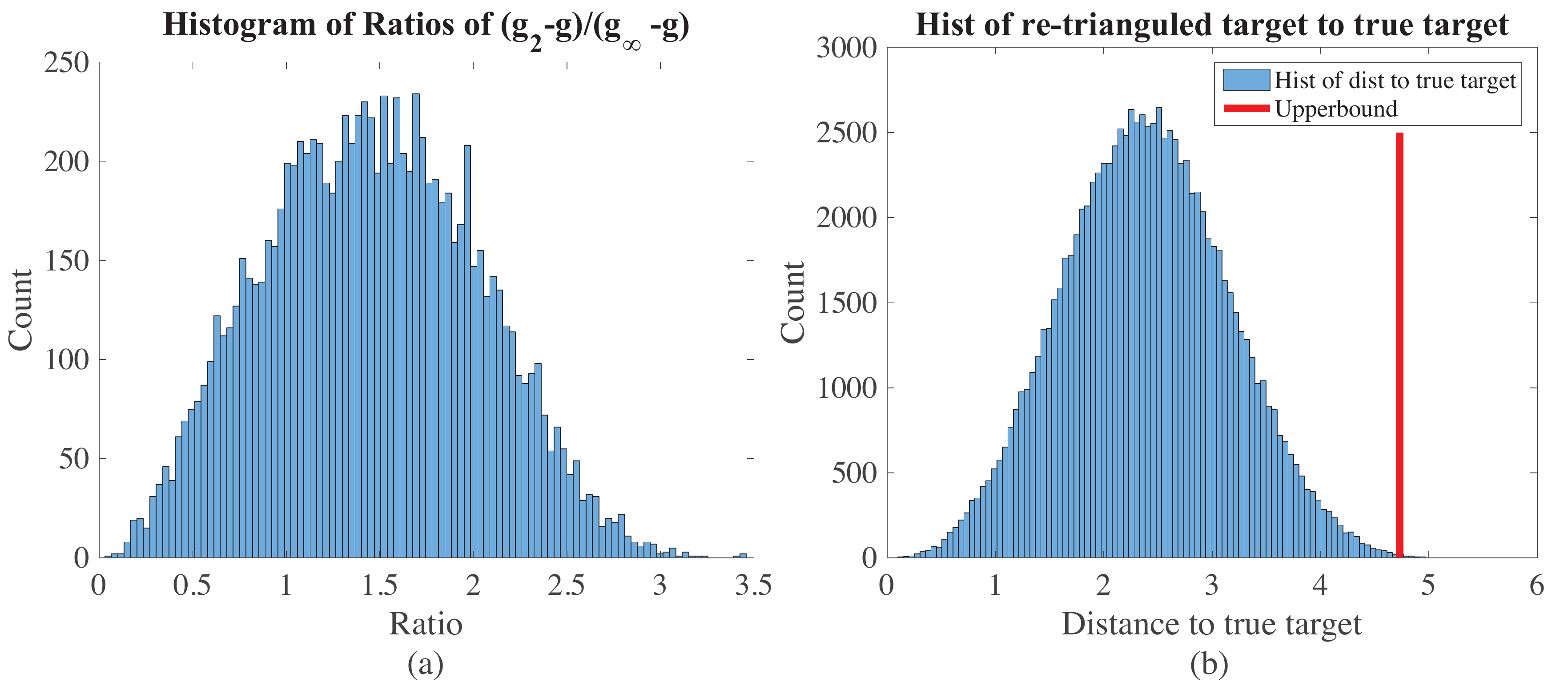}
	\caption{(a)~Distribution of  $\frac{|\hat{g}_2-g|}{|\hat{g}_\infty-g|}$ (b)~Histogram of the error: $|\hat{g}-g|$ with the following noise parameters: $|n_p| \leq 10$,$|n_s| \leq 0.1h$, $|n_\theta|\leq 1^\circ$}
	\label{fig:hist}
\end{figure}

The first two sources are less than one pixel. To investigate the role of camera orientation, we perturbed camera location
$\hat{s} = s + n_s$, where $s$ is the true location and $n_s$ is a uniform noise, and camera pose $\hat{\theta} = \theta + n_\theta$ where $\theta$ is the true orientation and $n_\theta $ is a uniform noise. Figure~\ref{fig:hist} (b) reports the result of triangulation error from two cameras in an optimal position. The height of the viewing plane was set to $10$m. The noise is set to $|n_p| \leq 10$, $|n_s| \leq 0.1h$, and $|n_\theta| \leq 1^\circ$. Each simulation was repeated $10^5$ times where the target location $\hat{g}$  is computed by triangulation and the error $|\hat{g}-g|$ is reported. Various noise levels are shown in the captions.
If we choose a bound of 10 pixels for the measurement error, it corresponds to  $\alpha < 0.1$ rad. The solid red line shows the predicted worst case error using our model. In general, reprojection error will be less than 10 pixels, otherwise it will be discarded as outliers. The state-of-the-art SLAM~\cite{zhang2015visual} algorithm's performance can go up to $0.0014 \ deg/m$ therefore, we set the camera position error to be less than $10 \%$ of the height while bounding the orientation error to be less than $1^\circ$. 
The histogram shows that the distance to the true target location is well bounded by the worst case uncertainty which is indicated as the vertical red line. It means that our uncertainty cone model can be relatively robust to system noises. 

Next, we study the effect of using two cameras vs. all cameras. 
We estimate the target pose using least squares from all cameras and report the ratio:
$\frac{|\hat{g}_2-g|}{|\hat{g}_\infty-g|}$ is plotted in Fig~\ref{fig:hist} (a). Here, $\hat{g}_2$ is the estimated target location using the optimal pair while $\hat{g}_\infty$ uses all the cameras. The simulation was repeated $10^4$ times. 
The ratio in Fig~\ref{fig:hist} (a) is less than 3.5, which means that using the optimal pair of cameras to triangulate the target is at most 3.5 times worse than triangulation using all cameras. This is because that the triangulation error using two or all camera views can be considered as a random process. Using only two camera views does not restrict the target as rigorous as using more all views, therefore, imposing at most 3.5 ratio of target position error.

\subsection{Real Experiment}
We collected two data sets using a GOPRO HERO 3 with a UAV flying over the same region with different height. The altitude ranges from 10 meters to 30 meters whereas the covered areas range between planar to more general orchard scenes. The orchard contains trees that are around 3 meters tall and ground elevation difference around 1 meter. We recorded around 5 minutes of videos, which is roughly 10000 frames. In order to speed up the reconstruction, we extracted every $30^{th}$ frames of the videos for mosaicking, which results in around 400 frames.
We used the commercial AgiSoft software for Structure from Motion for dense reconstruction and mosaicking to investigate the effect of view selection on reconstruction quality and reprojection error.

\subsubsection{Mosaic Quality}
\begin{figure}
\centering
	\includegraphics[width=0.4\textwidth]{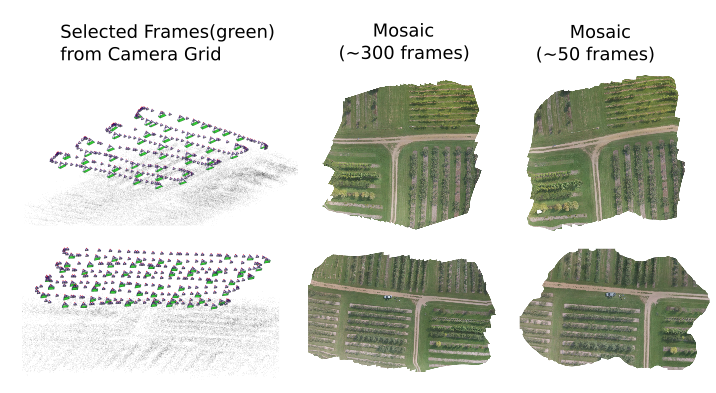}
	\caption{Select a subset of the original frames using the camera grid: reduces frames from $\sim$300 to $\sim$50 with comparable mosaic quality.}
\label{fig:mosaic}
\end{figure} 

We use the original selected frames for reconstruction and mosaicking~\cite{zhengqi}. Then, we use grid resolution of $\delta_d = h$ as shown in Figure~\ref{fig:mosaic} to select a subset of the frames for reconstruction and mosaicking. This means that if the drone is 10 meters above the ground plane, we select camera frames every 10 meters, which significantly reduces the number of cameras required. 
The total time required to reconstruct the same region decreased significantly while the reprojection error of each reconstruction remains low as shown in Table~\ref{tab:reconT}. For qualitative evaluation,  we stitched the images together by using the output pose from SFM and orthorectifying the views to compare the quality of the final mosaic. The resulting views are comparable, indicating that the proposed view selection mechanism does indeed perform comparable with respect to the original input set as shown in Figures~\ref{fig:mosaic}.

\subsubsection{Dense Reconstruction Quality}
We also examine the performance of the multi-resolution camera grid approach at the orchard data sets. 
For dense reconstruction, such data sets should be considered as a general scene and they cannot be treated as planar region, otherwise, features on different height cannot be covered.
We first use ORB-SLAM~\cite{murORB2} to extract camera poses and sparse point clouds. Since the point clouds still contains many inconsistent points, a filter is applied to remove noisy points too far from the surroundings. Then a mesh is built upon those points with maximum of 10,000 faces. We extracted the visibility cones of the mesh with the given trajectory and sampled a coarse-to-fine camera view grid in the same trajectory. The original data sets last around 5 minutes and contains more than 9000 images. Using the key frame selection method from ORB-SLAM, more than 3000 images are selected for reconstruction. It is unfeasible due to computational limitations. Therefore, we selected every $10^{th}$ frames with a total of around 900 frames. 
As shown in Figure~\ref{fig:viewselection}, the view selection algorithm selected a relatively sparser views in flat regions comparing to the densely packed views in more complex regions. The view selection algorithm will terminate when at least $95\%$ of the surfaces are covered. Therefore, there are still a few meshes that cannot be visible to the view subsets in the last iteration. The initial grid spacing is set to the height between the camera view plane and dominant ground plane: $\delta_d = h$. 
The reconstruction time and reprojection error comparison is shown in Table~\ref{tab:reconT}. It is clear that the computational time decreased by more than a magnitude and while the reprojection error does not increase too much. Essentially, our multi-resolution approach takes the scene geometry into consideration and removes redundant views that does not contribute much to the results.  
Visually, we can see that the dense reconstruction quality is very comparable shown in Figure~\ref{fig:denseComp30} taken from 30 meters above and in Figure~\ref{fig:denseComp10} taken from 10 meters above. Both results show that the reconstruction quality are almost identical.   
There is also an interesting observation: it is \emph{not} necessarily beneficial to have as many as views possible for dense reconstruction. As shown in Figure~\ref{fig:denseComp30} (a), more views actually smooth out the distinct geometry of the trees, leaving edges blending into each other. At a lower altitude, as shown in Figure~\ref{fig:denseComp10}, the dense reconstruction results are almost indistinguishable. 

\section{Conclusion}

In this paper, we studied view selection for a specific but common
setting where a ground plane is viewed from above from a parallel
viewing plane.  We showed that for a given world point, two views can be
chosen so as to guarantee a reconstruction quality which is almost as
good as one that can be obtained by using all possible views. Next, by
fixing these two views and studying perturbations of the world point,
we showed that one can put a coarse grid on the viewing plane and
ensure good reconstructions everywhere. Even though the reconstruction
quality can be improved by increasing the grid resolution, we showed
that a grid resolution proportional to the scene depth suffices to
guarantee a constant factor deviation from the optimal reconstruction.
We then showed how to extend the bound in the presence of perturbations of the viewing or scene planes.
However, as the scene geometry gets more sophisticated, occlusions
must be addressed.  For this purpose, we presented a multi-resolution view selection mechanism.
We also presented an application of these results to image mosaicking and scene reconstruction 
from (low altitude) aerial imagery.

Our results provide a foundation for multiple
avenues of future research.  An immediate extension is for scenes
which can be represented as surfaces composed of multiple planes.
Giving guarantees in the presence of occlusions 
 raises ``art gallery'' type research
problems~\cite{o1987art}.  Furthermore, rather than selecting views
apriori and in one shot, the view selection can be informed by the
reconstruction process as is commonly done in existing literature. Our
multi-resolution view selection method provides the starting point for a batch  scheme where a
coarse grid is used for reconstruction under the planar scene assumption
and further refined based on the intermediate reconstruction.

\newpage
\clearpage
\bibliographystyle{plainnat}
\bibliography{references}

\newpage
\clearpage
\appendix
\label{sec:appendix}

\section{More Reconstruction Results}
\begin{figure*}
\centering
	\includegraphics[width=1\textwidth]{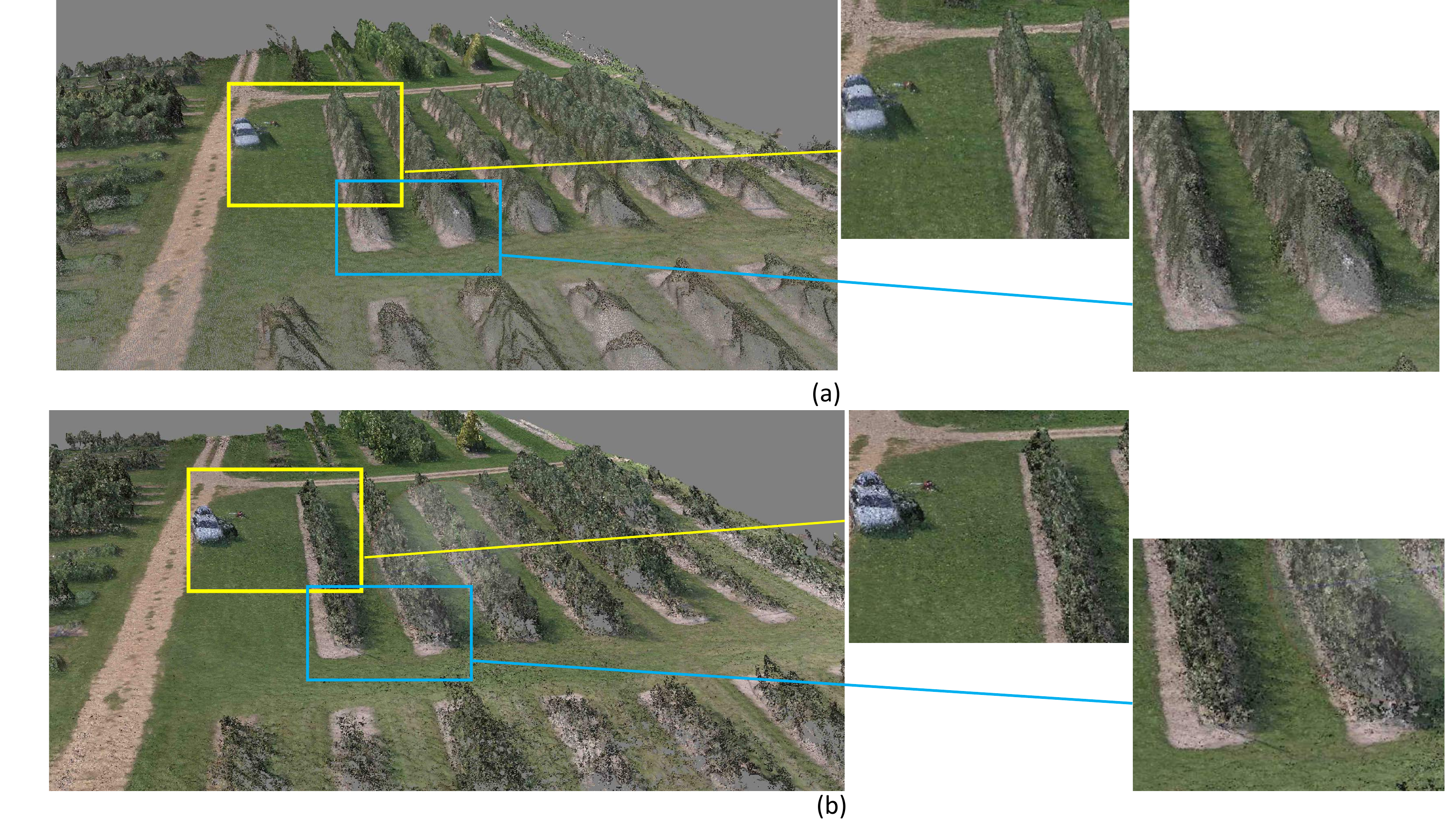}
	\caption{Comparison of dense reconstruction of the orchard taken at 30 meters altitude. (a) Dense Reconstruction using 875 images, with closeup views of the trees. (b) Dense Reconstruction using 209 images extracted using our multi-resolution view selection method, with closeup views of the trees. }
\label{fig:denseComp30}
\end{figure*}

\section{Lemmas and Theorems}
\subsection{Proof of Lemma~\ref{lem:diag1max}}
\begin{proof}
We prove the lemma by contradiction:
Suppose there exists a camera $s_k \in A-\{s_p,s_q\}$ such that
$Cone((s_k,\theta_k),g)$ intersects $\overline{v_1v_3}$ at point $u_1,u_2$, where $u_1 \leq v_1$ and $u_2 \geq v_3$ as shown in Fig~\ref{fig:campq} (b); 
Since $Cone((s_k,\theta_k),g)$ must contain target $g$, $u_1 = v_1$. 
We know that $u_1,u_2$ are on the vertical line passing through $g$, we can formulate $\overline{u_1v_2}$ using the law of sine of the triangle $\triangle(s_ku_1u_2)$.
\begin{align*}
\frac{\overline{u_1u_2}}{\sin(2\alpha)} &= \frac{h/\sin(\theta_k-\alpha)}{\sin(\pi/2-\theta_k-\alpha)} \\
\overline{u_1u_2} &= \frac{2h\sin(2\alpha)}{\sin(2\theta_k)-\sin(2\alpha)}
\end{align*}
Since $u_2 \geq v_3$, we want to find the minimum $\overline{u_1u_2}$ by choosing different $s_k \neq s_p,s_q$, which is equivalent to minimizing $\overline{u_1u_2}$ w.r.t. $\theta_k$. 
Thus, $\overline{u_1u_2}$ is minimized when $\sin(2\theta_k) = 1$, which results in $\theta_k = \pi/4$.
By substituting $\theta_k=\pi/4$, $\overline{v_1v_k} = \frac{2h\sin(2\alpha)}{1-\sin(2\alpha)} = diag_1$. It means that either $s_k = s_q$ or $\overline{u_1u_2} \geq \overline{v_1v_3}$, both of  which contradict with our assumption. 
\end{proof}

\subsection{Proof of Lemma~\ref{lem:upperbound2camerasValue2d}}
\begin{proof}
Using small angle approximation, we get $\sin(\alpha) \approx \alpha$ and $\cos(\alpha)\approx 1$ and $\alpha^2 \approx 0$.
The angles are constrained such that $\theta_p, \theta_q \in [\pi/4-2\alpha, \pi/4]$.

\begin{align*}
diag_1 &= ||r_1^2+r_2^2-2\cdot r_1 \cdot r_2 \cdot \cos(\theta_p+\theta_q)||_2 \\
		&\approx ||2t^2\alpha^2 + 2t^2\alpha^2 - 4t^2\alpha^2 \cos(\theta_p+\theta_q)||_2 \\
		&= 2t\alpha||1-\cos(\theta_p+\theta_q)||_2
\end{align*}
$\max(diag_1) \leq 2t\alpha$ and $\min(diag_1) \geq \sqrt{1-4\alpha} \cdot 2t\alpha$

\begin{align*}
diag_2 &= ||r_1^2+r_4^2-2\cdot r_1 \cdot r_4 \cdot \cos(\pi - \theta_p-\theta_q + 2\alpha)||_2 \\
		&\approx 2t\alpha||1+\cos(\theta_p+\theta_q)-2\alpha\sin(\theta_p+\theta_q)||_2
\end{align*}
$\max(diag_2) \leq \sqrt{1+2\alpha} \cdot 2t\alpha$ and $\min(diag_2) \geq \sqrt{1-4\alpha} \cdot 2t\alpha$.
Therefore,$diag_2 \leq \sqrt{\frac{1+2\alpha}{1-4\alpha}} diag_1$ and $1-4\alpha$ will not be negative since $\alpha$ must be less than $0.25$ to satisfy small angle approximation.
Given that $\varepsilon_2 = \max(diag1,diag2)$, we can conclude 
$$\varepsilon_2 \leq \frac{1+2\alpha}{1-4\alpha} \frac{2h\sin(2\alpha)}{1-\sin(2\alpha)}$$
\end{proof}

\subsection{Proof of Lemma~\ref{lem:grid2d1}}
\begin{proof}
We will add two more line segments $\overline{aa'}$ and $\overline{cc'}$ to generate a isosceles trapezoid $aa'cc'$ (Fig~\ref{fig:u3d2d}). When the angle $\angle{s_paa'} \geq \angle{s_pab}$, the diagonal $\overline{ac}$ will be the longest line segments in the trapezoid $aa'cc'$. Therefore, when $\angle{s_paa'} \geq \angle{s_pab}$, that is $\theta_p + \theta_q \geq \frac{\pi}{2} + \alpha$, is satisfied, $||diag_1|| > ||diag_2||$.
\end{proof}

\subsection{Proof of Lemma~\ref{lem:grid2d2}}
\begin{proof}
 First, when the inner half planes of both $Cone(s_p)$ and $Cone(s_q)$ intersect above $g$, it is clear that by moving the intersection down to $g$, $\theta_p + \theta_q$ is increased.
Now assume target $g$ is moving along the $x$ axis (Fig~\ref{fig:grid3dv2}) by some length $m$, where $m \leq \delta_d/2$. 
We can formulate $\theta_p + \theta_q$ as a function of $m$ and the distance between the cameras as 
\begin{align*}
f(m) = \theta_p + \theta_q = \tan^{-1}(\frac{h}{h/\tan(\pi/4-\alpha) - m}) + \\
 \tan^{-1}(\frac{h}{h/\tan(\pi/4-\alpha) + m}) + 2\alpha
\end{align*}
We can get the derivative $\frac{df(m)}{dm}$ as 
\begin{align*}
\frac{d}{dm} f(m) &= \{2m(2\cos(2\alpha)+2\cos(2\alpha)\sin(2\alpha)\} \cdot \\ &\{2m^2\sin(2\alpha)+2m^4\sin(2\alpha)+4m^2\sin^2(2\alpha) \\
&+m^4\sin^2(2\alpha)+m^4+4\}^{-1}\\
\end{align*}
Since $\frac{df(m)}{dm} \geq 0$, $\theta_p + \theta_q$ keeps increasing and is maximized  at target $g^* = g \pm \delta_d/2$.
\end{proof}

\subsection{Proof of Theorem~\ref{thrm:grid3d1}}
\begin{proof}
The intersection length $x$ is obtained using the law of sines. 
\begin{align*}
\frac{x}{\sin(2\alpha)} &= \frac{h/\sin(\theta-\alpha)}{\sin(\frac{\pi}{2} - \theta - \alpha)}
		x &=\frac{2h\sin(2\alpha)}{\sin(2\theta)-\sin(2\alpha)}
\end{align*}
When the inner half-plane of $Cone(s_p)$ and $Cone(s_q)$ intersect $g \pm \delta_d/2$, $x$ is maximized.
We can now compute directly the worst case uncertainty when $\alpha \leq 0.1$ rad which gives the desired result.
\end{proof} 

\subsection{Proof of Lemma~\ref{lem:grid2dhori}}
\begin{figure}[h]
\centering
	\includegraphics[width=0.4\textwidth]{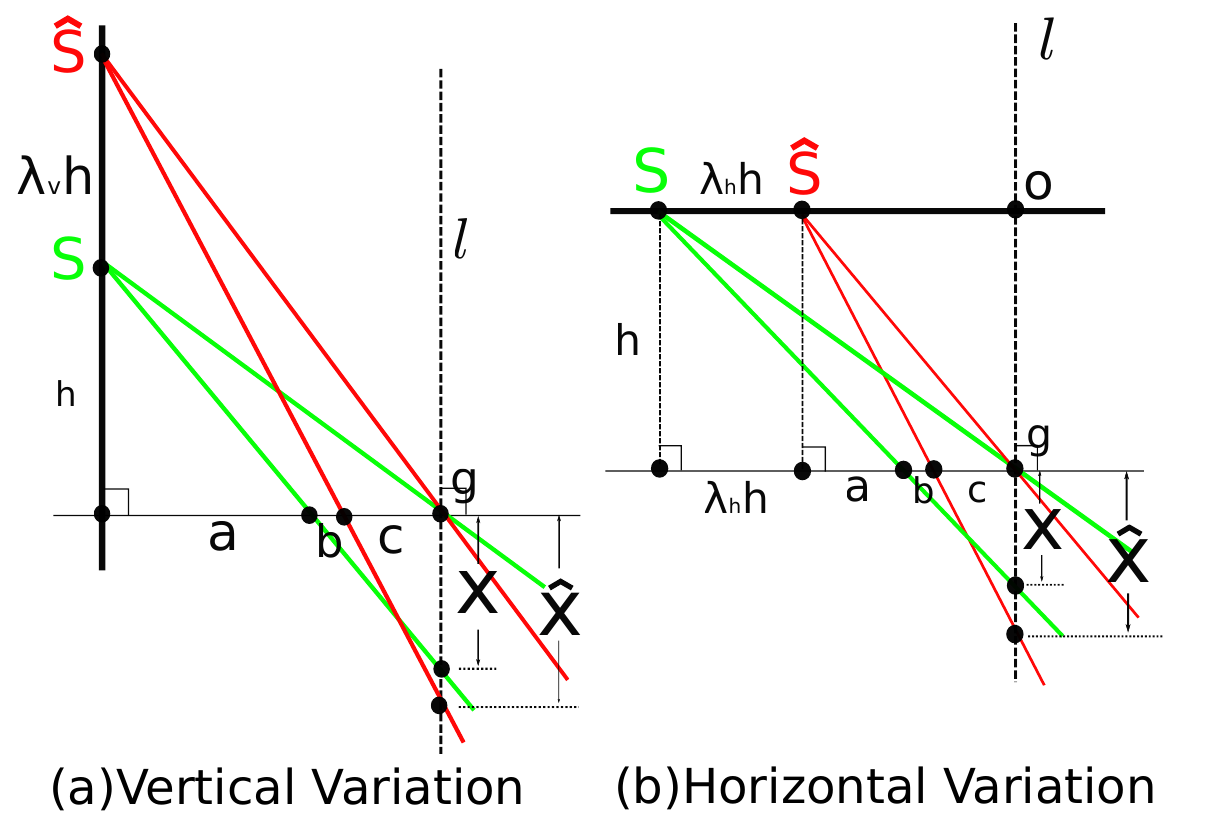}
	\caption{Variation in horizontal and vertical positions}
	\label{fig:variation}
\end{figure} 
First, we analyze the effects of horizontal variation $\lambda_h$. 

\begin{lemma}\label{lem:grid2dhori}
Let $s = (s_x,s_y)$ be a camera location in an optimal pair for target $g \in \overline{G}$. Let $\hat{s} = (s_x \pm \lambda_hh, s_y)$ obtained by perturbing $s$ in the horizontal direction. 
Let $\hat{x} = l \cap Cone(\hat{s})$ and $x = l \cap Cone(s)$. 
$$||\hat{x}|| \leq \frac{1}{1-\lambda_h} ||x||$$. 
\end{lemma}

\begin{proof}
From Lemma~\ref{lem:diag1max}, we can see that when sensor is at location $\hat{s} = (s_x + \lambda_hh,s_y)$, $||\hat{x}||$ is maximized. Therefore, $||\hat{x}|| \geq ||x||$. From Fig~\ref{fig:variation}, we can get the following relationship using similar triangles: 
$\frac{||x||}{b+c} = \frac{h}{\lambda_hh + a}$
 and $\frac{||\hat{x}||}{c} = \frac{h}{a+b}$. 
We can get the following result.
\begin{align*}
\frac{||\hat{x}||}{||x||} &= \frac{c(\lambda_hh+a)}{(a+b)(b+c)}  \leq \frac{\lambda_hh + a}{a+b}\\
& \leq \frac{h}{h-\lambda_hh} \leq \frac{1}{1-\lambda_h}
\end{align*}
\end{proof}

\subsection{Proof of Lemma~\ref{lem:grid2dvert}}

Then, we add vertical perturbation $\lambda_v h$ in between the viewing plane and the ground plane.
\begin{lemma}\label{lem:grid2dvert}
Let $s = (s_x,s_y)$ be a camera location in a optimal pair for target $g \in \overline{G}$. Let $\hat{s} = (s_x, s_y \pm \lambda_vh)$ obtained by perturbing $s$ in the vertical direction. 
Let $\hat{x} = l \cap Cone(\hat{s})$ and $x = l \cap Cone(s)$. 
$$||\hat{x}|| \leq (1+\lambda_v) ||x||$$
\end{lemma}

\begin{proof}
From Lemma~\ref{lem:diag1max}, we can see that when sensor is at location $\hat{s} = (s_x ,s_y+ \lambda_vh)$, $||\hat{x}||$ is maximized. Therefore, $||\hat{x}|| \geq ||x||$. From Fig~\ref{fig:variation}, we can get the following relationship using similar triangles:
$\frac{||x||}{b+c} = \frac{h}{a}$ and 
$\frac{||\hat{x}||}{c} = \frac{h + \lambda_vh}{a+b}$. 
We can get the following result.
\begin{align*}
\frac{||\hat{x}||}{||x||} &= \frac{ac(1+\lambda_v)}{(a+b)(b+c)}  \leq 1+\lambda_v
\end{align*}
\end{proof}

\section{Derivations}
\subsection{Wedge Intersection}
\label{sec:wedgeIntersection}
Using the law of sines over the triangle $s_pv_1v_2$, we get $\frac{r_1}{\sin(2\alpha)} = \frac{\overline{s_pv_1}}{\sin{\angle{s_pv_2s_q}}}$. We also have  $\angle{s_pv_2s_q} = \pi - 2\alpha-\angle{s_pv_1v_2} = \pi - 2\alpha - (\theta_p + \theta_q - 2\alpha) = \pi-\theta_p - \theta_q$. From  
$\triangle(s_pv_1s_q)$, we know that $\frac{\overline{s_pv_1}}{\sin(\theta_q-\alpha)} = \frac{t}{sin(\pi-\theta_p-\theta_q+2\alpha)}$. By combining both equations, we obtain:
$
r_1 = \frac{t\sin(\theta_q-\alpha)\sin(2\alpha)}{\sin(\theta_p+\theta_q-2\alpha)\sin(\theta_p+\theta_q)}
$
Using the same method, we have:
$
r_2 = \frac{t\sin(\theta_p+\alpha)\sin(2\alpha)}{\sin(\theta_p+\theta_q)\sin(\theta_p+\theta_q+2\alpha)}
$
From $\triangle(s_pv_3v_4)$, we get:
$
r_3 = \frac{t\sin(\theta_q+\alpha)\sin(2\alpha)}{\sin(\theta_p+\theta_q)\sin(\theta_p+\theta_q+2\alpha)}
$
Similarly, from $\triangle(s_qv_1v_4)$
$
r_4 =\frac{t\sin(\theta_p-\alpha)\sin(2\alpha)}{\sin(\theta_p+\theta_q-2\alpha)\sin(\theta_p+\theta_2)}
$

\end{document}